\documentclass[letterpaper,11pt]{amsart}
\usepackage[utf8]{inputenc} 
\usepackage{amsmath}
\usepackage{amsthm}
\usepackage{hyperref}       
\usepackage{url}            
\usepackage{booktabs}       
\usepackage{amsfonts}       
\usepackage{nicefrac}       
\usepackage{microtype}      
\usepackage{xcolor}         
\usepackage{ bbold }
\usepackage{bbm,bm}
\usepackage{listings}
\lstdefinelanguage{pseudocode}{
	morekeywords={for, if, every, then, input}
}
\lstdefinestyle{pseudocode}{
	language=pseudocode,
	keywordstyle=\bfseries
}

\hypersetup{colorlinks = true, linkcolor =blue, anchorcolor = red, citecolor = blue, urlcolor = blue}

\usepackage[all,cmtip]{xy}
\xyoption{frame}

\usepackage{tikz}
\usetikzlibrary{matrix, positioning}
\usepackage{mathtools}
\usepackage{verbatim}

\usepackage{algorithmicx}
\usepackage{algorithm}
\usepackage{algpseudocode}

\usepackage{enumitem}

\usepackage[portrait,letterpaper,margin=1.05in]{geometry}

\newtheorem{theorem}{Theorem}
\newtheorem{definition}[theorem]{Definition}
\newtheorem{lemma}[theorem]{Lemma}

\newtheorem{corollary}[theorem]{Corollary}

\newenvironment{proofof}[1]{ {\noindent \em Proof of #1.}\/}{\hfill\qed\bigskip}

\DeclareMathOperator{\diag}{diag}
\def\rank{\operatorname{rank}}

\newcommand{\rset}{\mathbb R}

\newcommand{\ep}{{\varepsilon}}
\newcommand{\pimin}{{\pi_{\min}}}

\let\R\Real

\def\set#1{\{ #1 \}}

\def\Set#1{\left\{ #1 \right\}}
\def\Abs#1{\left| #1 \right|}

\def\Norm#1{\left\| #1 \right\|}
\def\Paren#1{\left( #1 \right)}

\makeatletter
\def\Bigbar#1{\mathrel{\left|\vphantom{#1}\right.\n@space}}

\makeatother

\newcommand{\hset}{\mathbb H}

\newcommand{\suppress}[1]{}

\def\eps{\varepsilon}

\def\given{\mid}

\newcommand{\be}{\begin{equation}}
		\newcommand{\ee}{\end{equation}}
\newcommand{\bea}{\begin{eqnarray}}
		\newcommand{\eea}{\end{eqnarray}}
\newcommand{\bean}{\begin{eqnarray*}}
		\newcommand{\eean}{\end{eqnarray*}}

\def\DD{\mathcal{D}}

\def\tpose{\top}

\newcommand{\poly}{\mbox{poly}}

\def\tpose{\mathsf{T}}

\def\mm{{\mathbf{m}}}

\def\AA{{\mathbf{A}}}

\def\GG{{\mathbf{G}}}

\def\One{{\mathbb{1}}}
\def\bsig{{\bm{\sigma}}}

\DeclareMathOperator{\mom}{g}
\DeclareMathOperator{\emom}{{\tilde{g}}}

\def\Talg{{\mathcal{T}}}

\def\E{\operatorname{\mathbb{E}}}

\newcommand{\CC}{\mathbf{C}}
\newcommand{\UU}{\mathbf{U}}
\newcommand{\VV}{\mathbf{V}}
\newcommand{\WW}{\mathbf{W}}
\newcommand{\ZZ}{\mathbf{Z}}
\newcommand{\XX}{\mathbf{X}}
\newcommand{\YY}{\mathbf{Y}}
\newcommand{\eCC}{{\tilde{\mathbf{C}}}}

\newcommand{\hCC}{{\hat{\mathbf{C}}}}

\newcommand{\hUU}{\hat{\mathbf{U}}}
\newcommand{\hVV}{\hat{\mathbf{V}}}
\newcommand{\hWW}{\hat{\mathbf{W}}}
\newcommand{\hZZ}{\hat{\mathbf{Z}}}
\newcommand{\hSS}{\hat{\mathbf{S}}}
\newcommand{\hTT}{\hat{\mathbf{T}}}

\def\bSigma{\mathbf{\Sigma}}
\def\bDelta{\mathbf{\Delta}}

\DeclareMathOperator{\Vandermonde}{Vdm}

\def\eps{\varepsilon}
\def\given{\mid}

\newcommand{\nn}{{\mathbf{n}}}
\newcommand{\gb}{{\mathbf{g}}}
\newcommand{\model}{\mathrm{model}}
\newcommand{\stat}{\mathrm{stat}}
\newcommand{\g}{\gamma}
\newcommand{\ds}{{\mathbf{d}}}

\begin{document}

\title[Identification of Mixtures of Discrete Product Distributions]{Identification of Mixtures of Discrete Product Distributions in Near-Optimal Sample and Time Complexity}

\author{Spencer L. Gordon} \address{Engineering and Applied Science, California Institute of Technology, Pasadena CA 91125, USA. {\tt slgordon@caltech.edu}}
\author{Erik Jahn} \address{Engineering and Applied Science, California Institute of Technology, Pasadena CA 91125, USA. {\tt ejahn@caltech.edu}}
\author{Bijan Mazaheri}\address{Engineering and Applied Science, California Institute of Technology, Pasadena CA 91125, USA. {\tt bmazaher@caltech.edu}}
\author{Yuval Rabani}\address{The Rachel and Selim Benin School of Computer Science and Engineering, The Hebrew University of Jerusalem, Jerusalem 9190416, Israel. {\tt yrabani@cs.huji.ac.il}}
\author{Leonard J. Schulman}\address{Engineering and Applied Science, California Institute of Technology, Pasadena CA 91125, USA. {\tt schulman@caltech.edu}} 
\thanks{Research supported by NSF CCF-1909972 and CCF-2321079, by ISF grants 3565-21 and 389-22, and by BSF grant 2023607.}

\maketitle

\begin{abstract} 
We consider the problem of \emph{identifying,} from statistics, a 
distribution of discrete random variables $X_1,\ldots,X_n$ that is
 a mixture of $k$ product distributions. The best previous sample complexity for $n \in O(k)$ was $(1/\zeta)^{O(k^2 \log k)}$ (under a mild separation assumption parameterized by $\zeta$). The best known lower bound was $\exp(\Omega(k))$.
 
 It is known that $n\geq 2k-1$ is necessary and sufficient for identification.
We show, for any $n\geq 2k-1$, how to achieve sample complexity and run-time complexity $(1/\zeta)^{O(k)}$. 
We also extend the known lower bound of $e^{\Omega(k)}$ to match our upper bound across a broad range of $\zeta$.

Our results are obtained by combining (a) a classic method for robust tensor decomposition, (b) a novel way of bounding the condition number of key matrices called Hadamard extensions, by studying their action only on flattened rank-1 tensors.
\end{abstract}


\section{Introduction} \label{sec:intro} 

\subsection{The problem and our results.}
This paper resolves the sample and runtime complexity of identification of mixtures of product
distributions, a problem introduced almost thirty years ago in~\cite{KMRRSS94}, and further
studied in~\cite{CGG01,FM99,HallZ03,FOS08,CR08,tahmasebi2018identifiability,ChenMoitra19,gordon2021source}.
In this problem, an observer collects samples from the distribution over $n$ binary (or otherwise
drawn from a small finite set) random variables $X_1, X_2, \dots, X_n$. The samples are collected
from a mixture of $k$ distinct sub-populations. Conditional on the sub-population $U\in\{1,2,\dots,k\}$
from which a sample is drawn, the random variables are independent. The sub-population $U$ from 
which a sample is collected is a random variable chosen, for each sample of $n$ bits, independently 
according to the frequencies of the sub-populations in the entire population. However, the observer 
does not know these frequencies, and does not get an indication of the sub-population from which a 
sample was drawn; the observer only sees the values of the $n$ observable random bits. The goal
is to reconstruct the probabilistic model that generated the collected samples, namely, the frequencies
of the sub-populations and the conditional product distributions on $(X_1,X_2,\dots,X_n)\in\{0,1\}^n$.

Most of the above literature discusses the problem of {\em learning} the model. That is, the goal is to
produce some model with similar statistics (as measured, for instance, by KL-divergence) on the 
observables as the model that generated the samples. This does not necessarily guarantee similarity
in the parameter space of the models. In this paper, we focus on the stricter goal of 
{\em identifying} the model. That is, our goal is to produce a model whose parameters are sufficiently 
close to the true underlying model to generate also similar statistics on the observables. It is
known that identification is not always possible, as there exist some distributions on the observables that can be generated by more than one model. 
However, a mild condition of separability, namely that on each observable the
distribution is different among the sub-populations, guarantees identifiability 
information theoretically~\cite{tahmasebi2018identifiability}. We shall use $\zeta$ to denote the
minimum difference between sub-population distributions on an observable. This will be defined
precisely later. Note that identification also implies learning.

Our main result is an algorithm that identifies the parameters of any $\zeta$-separated mixture of 
product distributions using $(1/\zeta)^{O(k)}(1/\pimin)^{O(1)} (1/\eps)^2$ samples and runtime, 
up to additive error $\eps$. Here $\pimin$ is the minimum frequency of a sub-population. The
result holds if there are at least $2k-1$ \ $\zeta$-separated observables, which is known to be a 
necessary condition~\cite{Teicher61,Blischke64}. 
This result greatly improves upon the best previously known complexity for identification (and learning), that required 
$(1/\zeta)^{O(k^2 \log k)}(1/\pimin)^{O(\log k)}(1/\eps)^2$~\cite{gordon2021source} samples from
$3k-3$ \ $\zeta$-separated observables. Furthermore,
we show that the sample complexity
of identification (for constant $\eps$) is at least $(1/(k\zeta))^{\Omega(k)}$ (note that $\zeta\le\frac 1 k$ 
always). This generalizes the previously known lower bound of $\exp(-\Omega(k))$ that held only for
$\zeta = \Theta(1/k)$ ~\cite{RSS14}. Hence, our results are essentially optimal, both in terms of the number 
of $\zeta$-separated observables needed, and in terms of the sample and runtime complexity (excluding the 
case of $\zeta = \frac{1}{k^{1+o(1)}}$, where a small gap remains). 

For large $n$, $n = \omega(k)$, if
a subset of $2k-1$ 
$\zeta$-separated observables is known, 
then the runtime bounds pick up an additional factor of $n$ (to identify all the remaining observables). Otherwise, if 
the required subset exists but is not known, then the runtime picks up an additional factor of $n^{O(k)}$ (to enumerate over all possibilities). In both cases, the sample complexity increases by a factor of $\log n$. 

\subsection{Related work and motivation.}
The seminal work of~\cite{FOS08} solves the learning problem for general $k$ in sample and runtime
complexity $n^{O(k^3)}$. This was improved in~\cite{ChenMoitra19} to $k^{O(k^3)} n^{O(k^2)}$. (That 
paper also studied the problem of learning a ``mixture of subcubes'' of the hypercube, which is the special
case where each random bit $X_i$ is either fixed or uniformly distributed; for this case, they showed
sample and runtime complexity of $n^{O(\log k)}$.) The identification problem was first solved
in better complexity in~\cite{gordon2021source}, where it was shown how to identify a mixture of $k$ product distributions on $3k-3$ 
$\zeta$-separated variables with sample and runtime complexity of $(1/\zeta)^{O(k^2 \log k)}$. That paper
reduced the problem to a special case of identifying a mixture of $k$ distributions on independent and
{\em identically distributed} bits. The latter problem is solved using an elegant two-century-old method of 
Prony~\cite{Prony1795} coupled with a robustness analysis given in~\cite{gordon2020sparse}.

As pointed out in~\cite{FOS08} and elsewhere, the case of observables taking values in a
finite set reduces to the case of binary observables. Moreover, it is known~\cite{FanLi22}
that with polynomial overhead in the size of the range of the observables, the problem reduces 
to the case of identifying the conditional expectations of real-valued observables that are
independent conditional on the sub-population. Thus, we shall focus in this paper on this
real-valued version of the problem.

A concrete context in which identification of mixture distributions comes up is causal inference.
When data is drawn from multiple sources or sub-populations it is said to contain a {\em latent 
class}~\cite{allman2009identifiability}, which mirrors an unidentified mixture source. Standard 
procedures for identifying causal effects require considering the distributions within each latent 
class separately to control for potential confounding effects~\cite{Pea09}. Such an 
approach is generally impossible unless these within-source probability distributions can be
identified.

More broadly, the theory of causal inference relies at its core upon \emph{Bayesian networks} of 
random variables~\cite{pearl1985bayesian,SpGlS00,Pea09}; such a network imposes conditional independencies among random variables of the system. An important scenario is that several latent classes are subject to the same ``system mechanics'' (i.e., Bayesian network), but have different statistics. In this case, the problem of identifying the model is a far-reaching generalization of the problem of identifying mixtures of product distributions. There is a recent algorithm for the more general problem~\cite{GMRS23}; and that algorithm uses as an essential (and complexity-bottleneck) subroutine, any algorithm for identifying mixtures of product distributions. Thus the improvements of the present paper, carry over directly to that application.

\subsection{Our methods.}
We study the so-called Hadamard extensions that were also used to derive the bound in~\cite{gordon2021source}.
We give a new and much more powerful bound on the condition number of the Hadamard extensions. This bound 
alone would improve the sample complexity of the algorithm in~\cite{gordon2021source} to $(1/\zeta)^{O(k \log k)}$.
We gain further improvement as follows. Instead of reducing the problem to identifying mixtures of iid (synthetic) bits, we reduce 
the problem to a tensor decomposition problem, where the tensor components are guaranteed to be well-conditioned 
(which makes the decomposition unique). An algorithm for tensor decomposition in this setting was given thirty 
years ago in~\cite{Leurgans93} and later analyzed for robustness in~\cite{Goyal14, Bhaskara14}. (This algorithm 
can also be seen as a generalization of the matrix pencil method~\cite{HuaSarkar90}, applied to the iid case in~\cite{KKMMR18}.) Adapting the tensor decomposition algorithm to our setting and analyzing it using our 
new condition number bound, then yields the sample complexity of $(1/\zeta)^{O(k)}$.

\subsection{Comparison with the parametric case.} 
The literature on mixture models for parametric families (exponential distributions, Gaussians in $\rset$ or 
$\rset^d$, etc.) is even more extensive and older than for discrete mixture models. It is essential to realize 
a fundamental difference between the types of problems. In general, data is generated by (unseen) selection 
of a sub--population $j$ $(1\leq j \leq k)$, followed by (seen) sampling of $n$ independent samples 
\emph{from the $j$-th distribution}. In almost every parametric scenario (think e.g., of a mixture of $k$ 
Gaussians or exponential distributions on the line), \emph{$n=1$ is sufficient} in order to (in the limit of many repetitions) 
identify the model. This is fundamentally untrue in the non-parametric case; we have already mentioned 
that a lower bound of $n\geq 2k-1$ was shown in~\cite{RSS14}; this threshold for $n$ is called there the 
``aperture'' of the problem. To see, for starters, why the aperture must be larger than $n=1$, consider a 
single binary variable with $k=2$ equiprobable sources (i.e., $\Pr(U=0) = \Pr(U=1) =1/2$), one of which has 
$\Pr(X=1 \given U=0)=\frac{3}{4}$ and the other $\Pr(X=1 \given U=1)=\frac{1}{4}$.  If we see after each 
selection of a source only a single sample of $X$, it is impossible to distinguish between the above mixture
and a mixture in which $\Pr(X=1 \given U=0)= 1$ and $\Pr(X=1 \given U=1)=0$. 
With access to multiple independent samples 
from the \emph{same} source, however, we get empirical estimates of higher moments of the distribution, and at the critical aperture can 
identify the model.

\subsection{Organization.} 
Section~\ref{sec:prelim} formally states the identification problem for mixtures of product distributions and sets up the key mathematical objects needed for our work. Section~\ref{sec:alg} 
describes our algorithm and states our upper bounds on its sample complexity. Along with the algorithm pointers are provided to the main steps of the analysis 
in Sections~\ref{sec:cond},~\ref{apx: analysis}. In section~\ref{sec:lowerbounds} we prove lower bounds 
on the sample complexity of the identification problem. Finally, in section~\ref{sec:discussion} we discuss potential further directions of research on our topic.

\section{Results and preliminaries}\label{sec:prelim}

\subsection{The k-MixProd problem} 

Consider $n$ real, compactly supported random variables $X_1, \dots, X_n$ that are independent conditional on a latent random variable $U$ with range $[k] = \{1, \dots, k\}$. Given iid samples of the joint distribution of $(X_1, \dots, X_n)$, we want to identify the distribution of $U$, given by $\pi_j:=\Pr(U=j) \; (j\in[k])$, and the conditional expectations $\mm_{ij} = \E(X_i \mid U=j) \; (i \in [n], j \in [k])$. Hence, the model parameters for our problem are given by a vector $(\pi, \mm) \in \Delta^{k-1} \times \R^{n \times k}$, where $\Delta^{k-1}$ denotes the $(k-1)$-simplex. 

Set $X_S=\prod_{i \in S} X_{i}$, so $\E(X_S\mid U=j) = \prod_{i \in S} \mm_{ij}$. The mapping of the model to the statistics is then given by:
\begin{align} 
&\g_n: \Delta^{k-1} \times [0,1]^{n \times k} \to \rset^{2^{[n]}} \\
&\g_n(\pi, \mm)(S) = \E(X_S) =  \sum_{j=1}^k \pi_j \E(X_S\mid U=j) =
\sum_{j=1}^k\pi_j \prod_{i \in S} \mm_{ij}
 \label{model-to-stats} \end{align}
We drop the subscript $n$ and write $\g$ when $n$ is implied. The $k$-MixProd identification problem is to invert $\g_n$, i.e., to recover $(\pi_j)_{j \in [k]}$ and $(\mm_{ij})_{i\in [n],j\in[k]}$ (up to permuting the set $[k]$). This task is interesting in two versions, exact identification of $(\pi, \mm)$ from $\g_n(\pi, \mm)$ (i.e., from perfect statistics), and approximate identification of $(\pi, \mm)$ from noise-perturbed statistics $\emom$, i.e., from $\emom \in \rset^{\{0,1\}^n}$ that is close to $\g_n(\pi, \mm)$. To make the latter goal precise we need to specify metrics on the domain and range of $\g_n$. These are $L_\infty$ metrics, up to relabelings of the latent variable. ($S_k$ denotes the symmetric group on $k$ letters.)
\begin{align} d_\model ((\pi,\mm),(\pi',\mm'))
&:=\min_{\rho \in S_k} \max\{\max_j |\pi_j-\pi'_{\rho(j)}|, \max_{i,j} |\mm_{i,j}-\mm'_{i,\rho(j)}|\}\\
d_\stat(\mom,\mom')&:=\max_{S \subseteq [n]} |\mom(S)-\mom'(S)|.
\end{align} 

The mapping $\gamma_n$ is not everywhere injective, so 
the $k$-MixProd model identification problem is not always feasible. To guarantee identifiability we need to make the following assumptions:

\begin{enumerate}[label=(\alph*)]

\item ($\zeta$-separation) each variable $X_i$ is $\zeta$-\emph{separated}, i.e. $|\mm_{ij} - \mm_{ij'}| \geq \zeta$ for all $j\neq j' \in [k]$; \label{ass:1}
\item (non-degenerate prior) for each $j \in [k]$, we have $\pi_j \geq \pimin > 0$; \label{ass:2}
\item (sufficiently many observables) there are at least $n \geq 2k-1$ variables $X_i$. \label{ass:3}

\end{enumerate}

Let $\DD_{n,\zeta,\pimin}$ denote the space of the $k$-MixProd models with $n$ variables satisfying assumptions~\ref{ass:1} and~\ref{ass:2}. Formally:
\begin{align}
\DD_{n,\zeta,\pimin}=\{(\pi, \mm) \in \Delta^{k-1} \times [0,1]^{n \times k} \mid 
\min_j \pi_j \geq \pimin, \; \forall i \;
\min_{j\ne j'} |\mm_{ij} - \mm_{ij'}| \geq \zeta\}, \label{defn:DD}
\end{align}

Theorem~\ref{thm: main} shows (quite apart from its algorithmic content) that for $n \geq 2k-1$, if $(\pi,\mm)$ is a model in 
$\DD_{n, \zeta, \pimin}$, then any model whose statistics are close (in $d_\stat$) to those of $(\pi,\mm)$, must also be close to $(\pi,\mm)$ in $d_\model$.

We now introduce some mathematical concepts that will be needed for our work. 

\subsection{Hadamard extensions and related definitions}
Subsets of $[n]$ will typically be denoted by a capitalized variable in ordinary font: $S \subseteq [n]$. 
\begin{definition}
Given a matrix $\AA$ of any dimensions, let $\AA_{i*}$ denote the $i$'th row of $\AA$, and $\AA_{*j}$ the $j$'th column of $\AA$. Where clear from context we write $\AA_i$ instead of $\AA_{i*}$. 
For $S$ a set of rows, $\AA[S]$ denotes the submatrix of $\AA$ consisting of the rows in $S$.
\end{definition}

\begin{definition}[Hadamard product] \label{had-prod}
The \emph{Hadamard product} is the mapping $\odot: \rset^{[k]} \times \rset^{[k]}  \to \rset^{[k]}$ which,
 for row vectors $u=\Paren{u_1,\dotsc,u_k}$ and $v=\Paren{v_1,\dotsc,v_k}$, is given by
 $u \odot v \coloneqq (u_1v_1,\ldots,u_kv_k)$. Equivalently, using the notation $v_\odot = \diag(v)$,
the Hadamard product is $u \odot v = u \cdot v_\odot$. The identity element for the Hadamard product is the all-ones row vector $\One$. 
\end{definition}

\begin{definition}[Hadamard extension] For $\nn \in \rset^{n \times p}$,
    the \emph{Hadamard extension} of $\nn$, written
    $\hset(\nn)$, is the $2^{n} \times p$ matrix
    with rows $\hset(\nn)_S$ for all $S\subseteq [n]$, where, for $S=\{i_1,\ldots,i_\ell\}$,
    $\hset(\nn)_{S}= \nn_{i_1} \odot \cdots \odot \nn_{i_\ell}$; equivalently $\hset(\nn)_{S, j}=\prod_{i\in S} \nn_{ij}$. In particular $\hset(\nn)_\emptyset=\One$, and for all $i\in [n]$, $\hset(\nn)_{\{i\}} = \nn_i$.
\end{definition}
This construction first to our knowledge appeared (not under this name) in~\cite{ChenMoitra19}.

\begin{definition}
The singular values of a real matrix $\AA$ are denoted $\sigma_1(\AA)\geq \sigma_2(\AA) \geq \dotsm$. The $L_2 \to L_2$ operator norm is denoted $\Norm{\AA}$. The condition number of $\AA$ is denoted $\kappa(\AA) = \Norm{\AA} \cdot \Norm{\AA^{-1}}$. 
\end{definition}

\begin{definition}[Vandermonde matrix]\label{defn:Vdm} The Vandermonde matrix $\Vandermonde(m) \in \R^{k\times k}$ associated with a row vector $m \in \R^{k}$ has entries $\Vandermonde(m)_{ij} = (m_j)^i$ for $i\in \Set{0,1,\dotsc,k-1}$ and $j\in \Set{1,2,\dotsc,k}$. We also write $\Vandermonde(m,r)$ for the $r \times k$ matrix with entries  $\Vandermonde(m,r)_{ij} = (m_j)^i$.
\end{definition}

\subsection{Multilinear moments} \label{sec:multilin}
The data we obtain from our samples will be estimates of $\E[X_S] = \E[\prod_{i \in S}X_i]$ for all subsets $S\subseteq [n]$. We call these the \emph{multilinear moments} of the distribution, since they are multilinear in the rows $\mm_{i}$.
Observe that $\E[X_S] = \sum_{j} \pi_j \prod_{i\in S} \mm_{ij} =(\hset(\mm))_S \cdot \pi$, or equivalently, 
$\E[X_S] = (\mm_{i_1} \odot \mm_{i_2} \odot \dotsb \odot \mm_{i_s})\pi$ where $S = \set{i_1,i_2,\dotsc,i_s}$. Hence, the vector of statistics for a model $(\pi, \mm)$ is given by $\gamma(\pi, \mm) = \hset(\mm) \pi$. Observe that source identification is not possible if $\hset(\mm)$ has less than full column rank, i.e., if $\rank(\hset(\mm)) <k$, as then the mixing weights cannot be unique.

\begin{definition} \label{cts} Given disjoint sets $S, T \subseteq \{2, \dots, n\}$, define
\begin{align*}
    \CC_{ST} &= \hset(\mm[S]) \cdot \pi_{\odot} \cdot \hset(\mm[T])^\tpose,\\
    \CC_{ST, 1} &= \hset(\mm[S]) \cdot \pi_{\odot} \cdot \mm_{1 \odot} \cdot \hset(\mm[T])^\tpose
\end{align*} 
Note, for $A \subseteq S, B \subseteq T$, 
\begin{align*}
    (\CC_{ST})_{A, B} &= \g_n(\pi, \mm)(A \cup B), \\
    (\CC_{ST, 1})_{A, B} &= \g_n(\pi, \mm)(A \cup B \cup \{1\})  
\end{align*}

Consequently $\CC_{ST}$ and $\CC_{ST, 1}$ are observable, that is, every one of their entries is a statistic which the algorithm receives (a noisy version of) as input.
\end{definition}

\subsection{Tensor decomposition}
A matrix $\AA$ is rank $1$ if and only if it can be written as $\AA = \mathbf{uv}^\tpose$ for some vectors $\mathbf{u}, \mathbf{v}$. This concept can be generalized for tensors: 
\begin{definition}
    A $3$-way tensor $\Talg \in \R^{d_1 \times d_2 \times d_3}$ is said to be of rank $1$ if there exist vectors $u \in \R^{d_1}, v \in \R^{d_2}, z \in \R^{d_3}$ such that for all $i,j,k$:
    \[\Talg_{ijk} = u_i \cdot v_j \cdot z_k.\]
\end{definition}

Now, the rank of any tensor $\Talg$ can be defined as the minimum number of rank-1-tensors that sum up to $\Talg$. Equivalently, a tensor of rank $r$ has the following decomposition: 

\begin{definition}
    A $3$-way tensor $\Talg \in \R^{d_1 \times d_2 \times d_3}$ has a rank-$r$-decomposition if there exist matrices $\mathbf{U} \in \R^{d_1 \times r}, \mathbf{V} \in \R^{d_2 \times r}, \mathbf{Z} \in \R^{d_3 \times r}$ such that for all $i,j,k$:
    \[\Talg_{ijk} = \sum_{\ell = 1}^r \UU_{i\ell} \VV_{j\ell} \ZZ_{k_\ell}.\]
    We write $\Talg = [\UU, \VV, \ZZ]$ and call $\UU, \VV, \ZZ$ the factor matrices or tensor components of $\Talg$.
\end{definition}

In general, the rank-$r$-decomposition of a tensor $\Talg$ need not be unique, but a classical result of Kruskal~\cite{KRUSKAL197795} gives sufficient conditions for uniqueness. 

\begin{definition}
    The Kruskal rank of a matrix $\AA$ is the largest number $r$ such that any $r$ columns of $\AA$ are linearly independent. 
\end{definition}

\begin{theorem}[Kruskal~\cite{KRUSKAL197795}] \label{thm:Kruskal}
    The rank-$r$-decomposition of a three-way tensor $\Talg = [\UU, \VV, \ZZ]$ is unique up to scaling and permuting the columns of the factor matrices if 
    \[k_\UU + k_\VV + k_\ZZ \geq 2r+2,\]
    where $k_\UU, k_\VV, k_\ZZ$ denote the Kruskal rank of the matrices $\UU, \VV, \ZZ$ respectively. 
\end{theorem}

\section{The algorithm} \label{sec:alg}

\subsection{Reducing $k$-MixProd to tensor decomposition} \label{sec:tensdecomp}
To motivate our algorithm, we first discuss a way of solving the $k$-MixProd identification problem for $n \geq 2k-1$ given perfect statistics. Consider three disjoint sets $S, T, U \subseteq [n]$ of $\zeta$-separated observables, such that $|S| = |T| = k-1$ and $|U| = 1$. For convenience, we index rows so that $U = \{1\}$. Consider the vector of perfect statistics $\mom = \g_{2k-1}(\mm[S \cup T \cup \{1\}]) $ that corresponds to the observables in $S \cup T \cup \{1\}$. We can naturally view $\mom$ as a three-way tensor $ \Talg \in \R^{2^S \times 2^T \times 2^{\{1\}}}$ whose entries are given by $\Talg_{A, B, C} = \mom_{A \cup B \cup C}$ for $A \subseteq S, B \subseteq T, C \subseteq \{1\}$. Since we have 
\begin{align*}
    \Talg_{A, B, C} = \sum_{j = 1}^k \pi_j \prod_{i \in A \cup B \cup C} \mm_{ij} = \sum_{j=1}^k \hset(\mm[S])_{A, j} \cdot \hset(\mm[T])_{B, j} \cdot (\hset(\mm_1) \cdot \pi_{\odot})_{C, j}, 
\end{align*}
the tensor $\Talg$ can be decomposed as $\Talg = \left[\hset(\mm[S]), \hset(\mm[T]), \hset(\mm_1) \pi_{\odot}\right]$. It will follow as a ``qualitative'' corollary of Theorem \ref{thm: singular} that $\hset(\mm[S])$ and $\hset(\mm[T])$ have full column rank, which implies both matrices have Kruskal rank $k$. Moreover, $\zeta$-separation of $X_1$ implies that the matrix $\hset(\mm_1) \cdot \pi_{\odot}$ has Kruskal rank $2$. Hence, by Theorem \ref{thm:Kruskal}, the decomposition of $\Talg$ is unique. We deduce that identifying the model parameters $\mm$ and $\pi$ from perfect statistics (provided in $\Talg$) is equivalent to computing the unique tensor decomposition of $\Talg$. An efficient algorithm for computing tensor decomposition in this setting has first been given by \cite{Leurgans93} and later analyzed for stability by \cite{Goyal14}, \cite{Bhaskara14}. The idea is to first project the components of $\Talg$ down to their image, i.e. find matrices $\UU, \VV \in \R^{2^{k-1} \times k}$ such that $\UU^\tpose \hset(\mm[S])$ and $\VV^\tpose \hset(\mm[T])$ are invertible. Then, compute $\hCC_{ST} = \UU^\tpose \CC_{ST} \VV$ and $\hCC_{ST, 1} = \UU^\tpose \CC_{ST, 1} \VV$ (see Definition~\ref{cts}) from the given statistics. Now, the key observation is that the tensor components $\UU^\tpose \hset(\mm[S])$ and $\VV^\tpose \hset(\mm[T])$ can be found as the eigenvectors of $\hCC_{ST, 1} \hCC_{ST}^{-1}$ and $\hCC_{ST, 1}^\tpose (\hCC_{ST}^\tpose)^{-1}$ respectively. This is because 
\begin{align}
    \notag \hCC_{ST, 1} \hCC_{ST}^{-1} &= \UU^\tpose \hset(\mm[S]) \cdot \pi_\odot \cdot  \mm_{1 \odot} \cdot \hset(\mm[T])^\tpose \VV \cdot \left(\UU^\tpose \hset(\mm[S]) \cdot  \pi_\odot \cdot \hset(\mm[T])^\tpose \VV \right)^{-1}\\
    \notag &= \UU^\tpose \hset(\mm[S]) \cdot \pi_\odot \cdot  \mm_{1 \odot} \cdot \hset(\mm[T])^\tpose \VV \cdot \left(\hset(\mm[T])^\tpose \VV \right)^{-1} \cdot \pi_\odot^{-1} \cdot \left(\UU^\tpose \hset(\mm[S])\right)^{-1}\\
    &=\UU^\tpose \hset(\mm[S]) \cdot  \mm_{1 \odot} \cdot \left(\UU^\tpose \hset(\mm[S])\right)^{-1}, \label{eq:diagonalization1}
\end{align}
and
\begin{align}
    \notag \hCC_{ST, 1}^\tpose (\hCC_{ST}^\tpose)^{-1} &= \left(\UU^\tpose \hset(\mm[S]) \cdot \pi_\odot \cdot  \mm_{1 \odot} \cdot \hset(\mm[T])^\tpose \VV\right)^\tpose \cdot \left(\left(\UU^\tpose \hset(\mm[S]) \cdot  \pi_\odot \cdot \hset(\mm[T])^\tpose \VV\right)^\tpose\right)^{-1}\\
    \notag &= \VV^\tpose \hset(\mm[T]) \cdot \pi_\odot \cdot  \mm_{1 \odot} \cdot \hset(\mm[S])^\tpose \UU \cdot \left(\hset(\mm[S])^\tpose \UU \right)^{-1} \cdot \pi_\odot^{-1} \cdot \left(\VV^\tpose \hset(\mm[T])\right)^{-1}\\
    &=\VV^\tpose \hset(\mm[T]) \cdot  \mm_{1 \odot} \cdot \left(\VV^\tpose \hset(\mm[T])\right)^{-1}. \label{eq:diagonalization2}
\end{align}
In both cases, the eigenvalues of the matrices above are given by the entries of $\mm_1$. Crucially, $\zeta$-separation of $\mm_1$ allows us to match up the columns of $\UU^\tpose \hset(\mm[S])$ and $\VV^\tpose \hset(\mm[T])$ (and guarantees numerical stability). The original entries of $\hset(\mm)$ and $\pi$ can then be found from linear systems. In fact, notice that
\begin{align}
    \UU^\tpose \hset(\mm[S]) \pi = \UU^\tpose \mom[2^S], \label{eq:pi}
\end{align}
and for any row $\mm_i$ with $i \notin S$, we have 
\begin{align}
    \UU^\tpose \hset(\mm[S]) \pi_\odot \mm_i^\tpose = \UU^\tpose \mom(R \cup \{i\})_{R \subseteq S}. \label{eq:mm}
\end{align}
At this point, $\pi$ and $\mm_i$ are the only unknowns in the equations above.  We now give Algorithm~\ref{alg:mp algo}, which performs the identification procedure described above. 

\begin{algorithm} 
    \caption{(adapted from \cite{Leurgans93}, \cite{Bhaskara14}) Identifies a mixture of product distributions on $2k-1$ binary variables given the joint distribution.} \label{alg:mp algo}
     \begin{lstlisting}[style=pseudocode,mathescape,numbers=left,escapechar=|,columns=fullflexible,breaklines=true]
     Input: Two disjoint sets $S,T$ of $\zeta$-separated, binary observables $X_i$, each of cardinality $k-1$; a single $\zeta$-separated, binary observable that is not part of $S, T$, wlog $X_1$; vector $\emom\in \rset^{2^{S\cup T \cup \{1\}}}$ ($\emom(R)$ is the empirical approximation to $E(X_R)=\g(\pi, \mm)(R)$).
     Construct $\eCC_{ST}$ and $\eCC_{ST, 1}$ by $(\eCC_{ST})_{AB} = \emom(A  \cup B)$ and $(\eCC_{ST, 1})_{AB} = \emom(A  \cup B \cup \{1\})$ for $A\subseteq S, B \subseteq T$. |\label{alg:appr}|
     Set $\hUU \in \R^{2^{|T|} \times k}$ to be the top $k$ left singular vectors of $\eCC_{ST}$ and $\hVV \in \R^{2^{|S|} \times k}$ to be the top $k$ right singular vectors of $\eCC_{ST}$. |\label{alg:SVD}|
     $\hCC_{ST} \gets \hUU^\tpose \eCC_{ST} \hVV$, $\hCC_{ST, 1} \gets \hUU^\tpose \eCC_{ST, 1} \hVV$ |\label{alg:project}|
     Set $\hSS$ to be the eigenvectors of $\hCC_{ST, 1}(\hCC_{ST})^{-1}$ (sorted from highest eigenvalue to lowest). |\label{alg:start-decompose}|
     Set $\hTT$ to be the eigenvectors of $\hCC_{ST, 1}^{\tpose}(\hCC_{ST}^{\tpose})^{-1}$ (sorted from highest eigenvalue to lowest). |\label{alg:end-decompose}|
     $\tilde{\pi} \gets \hSS^{-1} \cdot \hUU^\tpose \Paren{\emom(R)_{R\subseteq S}}$ |\label{alg:solve for rest start line}|
     for every $i\in T \cup \{1\}$, $\tilde{\mm}_i \gets \Paren{(\emom(R\cup \Set{i}))_{R\subseteq S}}^{\tpose} \cdot \hUU \cdot {\Paren{\hSS^{\tpose}}}^{-1} \cdot \tilde{\pi}_{\odot}^{-1}$. 
     for every $i \in S$, $\tilde{\mm_i} \gets \Paren{(\emom(R \cup \Set{i}))_{R\subseteq T}}^{\tpose} \cdot \hVV \cdot {\Paren{\hTT^{\tpose}}}^{-1} \cdot \tilde{\pi}_{\odot}^{-1}$. |\label{alg:solve for rest end line}|
     \end{lstlisting}  
 \end{algorithm}

\emph{Some comments on the algorithm:} first, it is not necessary that the sets $S,T$ are of size $k-1$; typically, $\lceil \lg k \rceil$ will suffice. It is not even necessary that observables are $\zeta$-separated. These assumptions guarantee that $\sigma_k(\CC_{ST})$ is large, but the latter, along with the argument that this is w.h.p. reproduced for $\sigma_k(\eCC_{ST})$, is sufficient for the success of the algorithm. 

Second, we do not need to start with knowledge of $S,T$. Given $n$ variables of which an unknown subset of $2k-1$ variables are $\zeta$-separated, we can simply perform the algorithm for all possible choices of subsets $S, T$ of size $k-1$ (and an additional single row). Then, we choose the computed model whose statistics are closest to the observed statistics as the final output. This exhaustive search can increase the runtime by a factor of about $n^{2k}$; but actually all these complexities are only exponential in the \emph{actual} needed size of $S\cup T$, which as noted, for generic $\mm$ will be as small as $\lceil \lg k \rceil$, which makes the algorithm far more attractive in practice. 

Our main result essentially states that Algorithm~\ref{alg:mp algo} is robust to noise, i.e. still performs identification accurately when only presented with approximate statistics. 

\begin{theorem}\label{thm: main}
Let $n=2k-1$ and fix any $\ep \in (0, \zeta/2)$.
Let $(\pi,\mm) \in \DD_{n,\zeta,\pimin}$. If Algorithm \ref{alg:mp algo} is given approximate statistics $\emom$ on $(X_1,\ldots,X_n)$ as input, 
satisfying $d_\stat(\g(\pi,\mm),\emom) < \pimin^{O(1)} \zeta^{O(k)} \ep$, 
then in runtime $\exp(O(k))$ the algorithm outputs
$(\tilde{\pi},\tilde{\mm})$ s.t.\ \be d_\model((\pi,\mm),(\tilde{\pi},\tilde{\mm}))<\ep. \label{eq:output} \ee
Moreover, this output is essentially unique, in the sense that: any (not necessarily 
$\zeta$-separated) model $(\pi', \mm')$ with $d_\stat(\gamma(\pi, \mm),
 \gamma(\pi', \mm')) < \pimin^{O(1)} \zeta^{O(k)} \ep$ also satisfies~\eqref{eq:output}.
\end{theorem}
\begin{corollary} \label{cor:upperbd}
For $n=2k-1$ random variables $X_i$, sample complexity 
\[(1/\zeta)^{O(k)} (1/\pimin)^{O(1)} (1/\eps)^{2}\]
suffices to compute a model that w.h.p. satisfies~\eqref{eq:output}.
\end{corollary} 
\begin{corollary} \label{cor-n-more-3k-3}
Let the number of observables be $n\geq 2k-1$, and of these let some $2k-1$ be $\zeta$-separated. If this subset is known, then sample complexity
\[\log n \cdot (1/\zeta)^{O(k)} (1/\pimin)^{O(1)} (1/\eps)^{2}\]
and post-sampling runtime $n \cdot \exp(O(k))$ suffices to compute a model that w.h.p.\ satisfies~\eqref{eq:output}.
If the subset is not known, then the same sample complexity
and post-sampling runtime $n^{2k} \cdot \exp(O(k))$ suffices to compute a model that w.h.p.\ satisfies~\eqref{eq:output}.
\end{corollary} 
Corollary~\ref{cor:upperbd} follows from Theorem~\ref{thm: main} by standard Chernoff bounds and a union bound. The proofs of Theorem~\ref{thm: main} and Corollary~\ref{cor-n-more-3k-3} are in Section~\ref{apx: analysis}.

\subsection{Annotated steps of the algorithm}
This section describes what each non-trivial line of the algorithm (Alg.~\ref{alg:mp algo}) accomplishes, and in each case points to the part of the analysis necessary to justify it.
We first set some definitions for the analysis. Let $n=2k-1$, let $\{2, \dots, n\}$ be the disjoint union of sets $S,T$ each of size $k-1$. Let $\mom=\g_n(\pi, \mm)$ and let $\emom\in \rset^{\{0,1\}^n}$ be the empirical statistics. The analysis relies on assuming that $(\pi, \mm) \in \DD_{n,\zeta,\pimin}$. In the sequel let \be \ds:=d_\stat(\emom,\mom). \label{defn:ds} \ee

Now for the line-by-line:
\begin{itemize}

\item Line~\ref{alg:SVD}: 
The SVD of $\eCC_{ST}$ can be computed with high numerical accuracy in time $\exp(O(k))$ using, for instance, the Golub-Kahan-Reinsch algorithm (see Lemma~\ref{lem:svd} and \cite{GvL-4}, chapter 8). The span of the top $k$ left and right singular vectors of $\eCC_{ST}$ approximate the images of $\hset(\mm[S])$ and $\hset(\mm[T])$ respectively. 

\item Line~\ref{alg:project}:
We project on the top $k$ singular vectors of $\eCC_{ST}$ to get an invertible matrix $\hCC_{ST}$. Lemma \ref{lem:interbds} shows that $\hCC_{ST}$ is well-conditioned, given that $\ds$ is small enough. Note that the two matrices $\hCC_{ST}$ and $\hCC_{ST, 1}$ can be arranged to a $k \times k \times 2$-tensor $\hat{\Talg}$ that is close to the tensor $\Talg = \left[\hUU^\tpose \hset(\mm[S]), \hVV^\tpose \hset(\mm[T]), \hset(\mm_1) \pi_\odot\right]$ by Lemma~\ref{lem:req}.

\item Lines \ref{alg:start-decompose},\ref{alg:end-decompose}: 
These two lines implement the core of the tensor decomposition algorithm from \cite{Leurgans93}, as explained in the previous section. Hence, the matrices $\hSS$ and $\hTT$ approximate the tensor components $\hUU^\tpose \hset(\mm[S])$ and $\hVV^\tpose \hset(\mm[T])$ of $\Talg$. The diagonalization steps can be performed with high numerical accuracy in time $\exp(O(k))$ using, for instance, the algorithm from \cite{diagonalization}. For the error bounds, we rely on Theorem~\ref{thm:decompose} from \cite{Bhaskara14}.

\item Lines \ref{alg:solve for rest start line}-\ref{alg:solve for rest end line}: Here, we solve for the model parameters $\pi$ and $\mm$, using the equations~\eqref{eq:pi} and~\eqref{eq:mm}. For error control, we apply Lemma~\ref{lem:linearsystems}.
\end{itemize}

\section{The condition number bound} \label{sec:cond} 
The key to the sample-complexity and runtime bounds for our algorithm lies in the following condition number bound for the Hadamard Extension.

\begin{theorem}\label{thm: singular} 
\begin{enumerate} \item \label{sing1}
Let $\mm$ consist of $k-1$ \, $\zeta$-separated rows in $\rset^k$, and observe that the singular values of 
$\hset(\mm) = \hset(\mm)$ satisfy $\sigma_1(\hset(\mm))\geq \ldots \geq \sigma_k(\hset(\mm)) \geq 0=\sigma_{k+1}(\hset(\mm))=\ldots=\sigma_{2^{k-1}}(\hset(\mm))$. Then
\be \sigma_k(\hset(\mm)) > \frac{1}{\sqrt{k}} \left(\frac{\zeta}{2\sqrt{5}}\right)^{k-1} =: \bsig. \label{bsig} \ee
\item \label{sing2} Let $(\pi,\mm)\in \DD_{2k-1,\zeta,\pimin}$ and let $\CC_{ST}$ be as in Defn.~\ref{cts}. Then
\[ \sigma_k(\CC_{ST}) > \pimin \bsig^2= \frac{\pimin}{k} \left(\frac{\zeta}{2\sqrt{5}}\right)^{2k-2}. \]
\end{enumerate}
\end{theorem}

\begin{definition} $\hset_p=\{\hset(\nn): \nn \in \rset^{[k-1] \times [p]}\}$. (So $\hset_1$ consists of rank-1 tensors of order $k-1$.) \end{definition}

The proof of Theorem~\ref{thm: singular} relies on the following insight. Since $\hset(\mm)$ has dimensions $2^{k-1} \times k$, $\sigma_k(\hset(\mm))$ characterizes the least norm of $\hset(\mm) \cdot v$ ranging over any unit vector $v$ (all norms in $L_2$), but it does \emph{not} characterize the least norm of vectors of the form $h^\tpose \cdot \hset(\mm)$, which is $0$ as the left-kernel of $\hset(\mm)$ is of course very large. The insight is that it \emph{does} become possible to bound $\sigma_k(\hset(\mm))$ in terms of such vectors $h$, \emph{provided $h$ is restricted to rank $1$ tensors.} With this in mind we define:
 \[ \tau(\mm)=\min_{0\neq h\in \hset_1} \Norm{h^\tpose \cdot \hset(\mm)}/\Norm{h}. \]
\begin{proofof}{Theorem \ref{thm: singular}} 
To show Part~\ref{sing2} from Part~\ref{sing1}: The SVD implies there is a $k$-dimensional space $V$ s.t.\ $\forall v\in V$, $\|v^\tpose \cdot \hset(\mm[S])\| \geq \bsig \|v\|$. Further, for all $w\in\rset^k$, $\|w^\tpose \cdot \pi_\odot\|\geq \pimin \|w\|$. And for all $w\in\rset^k$, $\| w^\tpose \cdot \hset(\mm[T])^\tpose \| \geq \bsig \|w\|$. So $\forall v\in V$, $\|v\cdot \CC_{TS}\| \geq \pimin \bsig^2 \|v\|$. 

In order to establish Part~\ref{sing1} we prove the following two lemmas. 
\begin{lemma}\label{lem:7} $\sigma_k(\hset(\mm)) \ge \tau(\mm) / \sqrt{k}$. \end{lemma}
\begin{lemma} \label{lem:8} $\tau(\mm) > (\zeta/2\sqrt{5})^{k-1}$. \end{lemma}
\begin{proofof}{Lemma~\ref{lem:7}}
Consider $v\in \rset^k$, $\|v\|=1$, achieving $\sigma_k(\hset(\mm))$, i.e.,
$r\coloneqq \hset(\mm) \cdot v$ satisfies 
$\Norm{r}=\sigma_k(\hset(\mm))$. W.l.o.g.\ the order of coordinates is such that $|v_k|\geq 1/\sqrt{k}$. 
The last column of $\hset(\mm)$ is then:
\be \hset(\mm)_{*k}=\frac{1}{v_k} \Paren{r-\sum_{j=1}^{k-1} v_j \hset(\mm)_{*j}}. \label{Hmk} \ee
 
Now we carefully choose $h \in \hset_1$ based on $\mm$. Define the column vector $\nn \in \rset^{[k-1]}$ by $\nn_i\coloneqq-1/\mm_{ii}$ ($1 \leq i\leq k-1$); and let 
\be h \coloneqq \Paren{\prod_1^{k-1} \mm_{ii}} \hset(\nn)
= (\mm_{11}; -1) \odot \ldots \odot (\mm_{k-1,k-1}; -1). \label{hselect} \ee

For $j \neq k$ we have:
\begin{eqnarray*}
h^\tpose \cdot \hset(\mm)_{*j} & = &  \sum_S (-1)^{|S|} \Paren{\prod_{i \notin S} \mm_{ii}} \Paren{\prod_{i \in S} \mm_{ij}}
 =  \prod_{i=1}^{k-1} (\mm_{ii}-\mm_{ij}) = 0. \end{eqnarray*}
So, $h^\tpose \cdot \hset(\mm)_{*j} = 0$ for $j = 1,\dotsc,k-1$. 
For $j=k$, we apply~\eqref{Hmk} to evaluate $ h^\tpose \cdot \hset(\mm)_{*k}$:
\[ \Paren{h^\tpose \cdot \hset(\mm)}_k
 =  \frac{1}{v_k} h^\tpose \cdot \Paren{r-\sum_{j=1}^{k-1} v_j \hset(\mm)_{*j}}
 =  \frac{1}{v_k} \left(h^\tpose \cdot r - \sum_{j=1}^{k-1} v_j h^\tpose \cdot \hset(\mm)_{*j}\right)
 =  \frac{1}{v_k} h^\tpose \cdot r. \]
The norm of $h^\tpose \cdot \hset(\mm)$ is then upper bounded by
\begin{align*}
 \Norm{h^\tpose \cdot \hset(\mm)}
& =  \Abs{(h^\tpose \cdot \hset(\mm))_k} \\ 
&= \frac{1}{v_k}h^\tpose \cdot r \\
& \leq \sqrt{k} \Norm{h} \Norm{r} \\
& = \sqrt{k}\Norm{h}\sigma_k(\hset(\mm)).
\end{align*}
\end{proofof}

\begin{proofof}{Lemma~\ref{lem:8}}
Consider any $\GG \in \hset_1$, say $\GG = \hset(\gb)$, $\gb\in\rset^{[k-1]}$. Then
$(\GG^\tpose \cdot \hset(\mm))_j
 = \sum_S \GG_S \hset(\mm)_{S,j}
= \prod_1^{k-1} (1+\gb_i \mm_{i,j})$.
We also note that $\|\GG\|=\sqrt{\prod_1^{k-1} (1+\gb_i^2)}$.
    
We now show that there is some $j$ such that $\prod_{i=1}^{k-1} \Abs{ \frac{1+\gb_i \mm_{ij}}{\sqrt{1+\gb_i^2}} }$ is large.
First, for any $i$ for which $\gb_i\geq \frac 1 2$, there is at most one $j$ s.t.\ $\mm_{ij}\leq \zeta$; exclude these $j$'s. Next, for each $i$ for which $\gb_i<\frac 1 2$, there is at most one $j$ s.t.\ $\Abs{\frac{1}{\gb_i}+\mm_{ij}}\leq \zeta/2$; exclude these
$j$'s. For the remainder of the argument fix any $j$ which has not been excluded. Since $\mm$ has $k$ columns while $\gb \in \rset^{k-1}$, such a $j$ exists. 
We now lower bound $\Abs{ \frac{1+\gb_i \mm_{ij}}{\sqrt{1+\gb_i^2}}} $ for each $i$; there are three cases. 
\begin{enumerate}[nosep]
\item $\gb_i\geq 1/2, \mm_{ij} > \zeta$. Then $\Abs{\frac{1+\gb_i \mm_{ij}}{\sqrt{1+\gb_i^2}}} \geq \mm_{ij} > \zeta$.
\item $-1/2 < \gb_i< 1/2$. Then $\Abs{\frac{1+\gb_i \mm_{ij}}{\sqrt{1+\gb_i^2}}} > \sqrt{\frac{(\mm_{ij}-2)^2}{5}} \geq 1/\sqrt{5}$.
\item $\gb_i \leq -1/2, |\frac1{\gb_i}+\mm_{ij}|> \zeta/2$. Then $\left| \frac{1+\gb_i \mm_{ij}}{\sqrt{1+\gb_i^2}}\right| =\left| \frac{\gb_i (\frac1{\gb_i}+\mm_{ij})}{\gb_i \sqrt{1+1/\gb_i^2}}\right|
  > \frac{\zeta}{2\sqrt{5}}$.
\end{enumerate}
We therefore have $\tau(\mm) > (\zeta/2\sqrt{5})^{k-1}$.
\end{proofof}

Part~\ref{sing1} is an immediate consequence of the two lemmas. \end{proofof}

\section{Analysis of the algorithm} \label{apx: analysis}
To begin with, we will need the following result on the numerical accuracy of computing the SVD:

\begin{lemma}[\cite{GvL-4}, section 5.4.1]\label{lem:svd}
    Given a matrix $\mathbf{A} \in \R^{m \times n}$ and precision $\eps > 0$, the Golub-Kahan-Reinsch algorithm computes an approximate SVD given by $\UU \in \R^{m \times m}, \mathbf{\Sigma} \in \R^{m \times n}, \VV \in \R^{n \times n}$ such that 
    \begin{itemize}
        \item $\mathbf{\Sigma}$ is a diagonal matrix;
        \item $\UU = \WW + \mathbf{\Delta U}$ and $\VV = \ZZ + \mathbf{\Delta V}$, where $\WW, \ZZ$ are unitary and $||\mathbf{\Delta U}||, ||\mathbf{\Delta V}|| < \eps$;
        \item $\WW \bSigma \ZZ^\tpose = \mathbf{A} + \mathbf{\Delta A}$ with $||\mathbf{\Delta A}|| < \eps ||\mathbf{A}||$.
    \end{itemize}
\end{lemma}

To analyze Algorithm~\ref{alg:mp algo}, let $\ds = d_\stat(\emom,\mom) = \max\{||\eCC_{ST} - \CC_{ST}||_\infty, ||\eCC_{ST, 1} - \CC_{ST, 1}||_\infty\}$ as in Definition~\eqref{defn:ds}. Furthermore, assume that the SVD in line~\ref{alg:SVD} of the algorithm is calculated with precision $\eps = \ds$ according to the statement of Lemma~\ref{lem:svd}. Then, we get the following bounds: 

\begin{lemma}\label{lem:interbds}
    \begin{enumerate}
        \item[(a)] $\sigma_k(\eCC_{ST}) \geq \pimin \bsig^2 - 2^k \ds $
        \item[(b)] $\sigma_k(\hCC_{ST}) \geq \pimin \bsig^2 - 2^{k+2} \ds$
    \end{enumerate}
\end{lemma}

\begin{proof}
    (a) By a classical perturbation bound of Weyl~\cite{Weyl12}, we have $|\sigma_k(\eCC_{ST}) - \sigma_k(\CC_{ST})| \leq ||\eCC_{ST} - \CC_{ST}||$. Moreover, we know $||\eCC_{ST} - \CC_{ST}|| \leq 2^k ||\eCC_{ST} - \CC_{ST}||_\infty \leq 2^k\ds$. Combining this with Part~\ref{sing2} of Theorem~\ref{thm: singular}, we get the desired result. 
    
    (b) Let $\UU \bSigma \VV^\tpose$ be an approximate SVD for $\eCC_{ST}$ satisfying the conditions of Lemma~\ref{lem:svd} with precision $\eps$. In particular, $\UU = \WW + \mathbf{\Delta U}$, $\VV = \ZZ + \mathbf{\Delta V}$ and $\WW \bSigma \ZZ^\tpose = \eCC_{ST} + \mathbf{\Delta} \eCC_{ST}$. We assume that the columns of $\bSigma$ are ordered by magnitude of their diagonal entries. Let $\hUU, \hVV, \hWW, \hZZ, \bDelta \hUU, \bDelta \hVV$ denote the first $k$ columns of the respective matrices. We get 
    \begin{align*}
        \sigma_k(\hCC_{ST}) &= \sigma_k(\hUU^\tpose \eCC_{ST} \hVV)
        = \sigma_k((\hWW + \bDelta \hUU)^\tpose \eCC_{ST} (\hZZ + \bDelta \hVV))\\
        &\geq \sigma_k(\hWW^\tpose \eCC_{ST} \hZZ) - ||(\bDelta \hUU)^\tpose \eCC_{ST} \hZZ|| - ||\hWW^\tpose \eCC_{ST} \bDelta \hVV|| - ||(\bDelta \hUU)^\tpose \eCC_{ST} \bDelta \hVV||\\
        &\geq \sigma_k(\hWW^\tpose \eCC_{ST} \hZZ) - 3 \eps ||\eCC_{ST}|| \\
        &\geq \sigma_k(\hWW^\tpose (\eCC_{ST} + \bDelta \eCC_{ST}) \hZZ) - ||\hWW^\tpose \bDelta \eCC_{ST} \hVV|| - 3 \eps ||\eCC_{ST}||\\
        &\geq \sigma_k((\eCC_{ST} + \bDelta \eCC_{ST})) - 4 \eps ||\eCC_{ST}||\\
        &\geq \sigma_k(\eCC_{ST}) - 5\eps ||\eCC_{ST}||.
    \end{align*}
    Since the entries of $\eCC_{ST}$ are all at most $1$, we have $||\eCC_{ST}|| \leq 2^{k-1}$ and with $\eps = \ds$, the result follows. 
\end{proof}

At this point, we can view $\hCC_{ST}$ and $\hCC_{ST,1}$ as the two slices of a $k \times k \times 2$-tensor $\hat{\Talg}$ that is close to the tensor $\Talg = [\hUU^\tpose \hset(\mm[S]), \hVV^\tpose \hset(\mm[T]), \hset(\mm_1) \cdot \pi_\odot]$ (with the two slices $\Talg(:,:,0) = \hUU^\tpose \CC_{ST} \hVV$ and $\Talg(:,:,1) = \hUU^\tpose \CC_{ST, 1} \hVV$). The following lemma bounds the distance between $\hat{\Talg}$ and $\Talg$, and shows that the first two components of $\Talg$ are full rank and well-conditioned.

\begin{lemma} \label{lem:req}
    Suppose that $\ds \leq \pimin \bsig^2 / (k2^{2k+2})$, then 
    \begin{enumerate}
        \item[(a)] $||\hCC_{ST} - \hUU^\tpose \CC_{ST} \hVV||_\infty, ||\hCC_{ST, 1} - \hUU^\tpose \CC_{ST, 1} \hVV||_\infty < 2^{2k} \ds$
        \item[(b)] $\kappa(\hUU^\tpose \hset(\mm[S])), \kappa(\hVV^\tpose \hset(\mm[T])) \leq k^22^{2k+1}/(\pimin \bsig^2). $
    \end{enumerate}
\end{lemma}

\begin{proof}
    (a) Remember that $\hUU = \hWW + \bDelta \hUU, \hVV = \hZZ + \bDelta \hVV$, where $||\bDelta \hUU||, ||\bDelta \hVV|| < \ds$ and the columns of $\hWW, \hZZ$ are orthonormal. In particular, this implies that entries of $\hUU$ and $\hVV$ are bounded by $1 + k \ds \leq 2$. Hence, we have 
    \begin{align*}
        &||\hCC_{ST} - \hUU^\tpose \CC_{ST} \hVV||_\infty = ||\hUU^\tpose(\eCC_{ST} - \CC_{ST})\hVV||_\infty\\
        \leq &||\hUU^\tpose||_\infty \cdot 2^{k-1} \cdot ||\eCC_{ST} - \CC_{ST}||_\infty \cdot ||\hVV||_\infty \cdot 2^{k-1} \leq 2^{2k}\ds.
    \end{align*}
    The result follows analogously for $\hCC_{ST, 1}$.
    
    (b)  First, since the entries of $\hset(\mm[S])$ are bounded by $1$, we have $\sigma_1(\hUU^\tpose \hset(\mm[S])) \leq k \cdot ||\hUU^\tpose \hset(\mm[S])||_\infty \leq k2^k$, and similarly, $\sigma_1(\hset(\mm[T])^\tpose \hVV) \leq k2^k$. Using part (a), Lemma~\ref{lem:interbds}, and the assumption on $\ds$, we get
    \begin{align*}
        \sigma_k(\hUU^\tpose \CC_{ST} \hVV) &\geq \sigma_k(\hCC_{ST}) - ||\hCC_{ST} -  \hUU^\tpose \CC_{ST} \hVV|| \geq \sigma_k(\hCC_{ST}) - k2^{2k} \ds\\
        &\geq \pimin \bsig^2 - 2^{k+2}\ds - k2^{2k} \ds \geq \pimin \bsig^2 / 2
    \end{align*}
    Hence, we deduce
    \begin{align*}
        \pimin \bsig^2/2 &\leq \sigma_k(\hUU^\tpose \hset(\mm[S]) \pi_\odot \hset(\mm[T])^\tpose \hVV)\\
        &\leq \sigma_k(\hUU^\tpose \hset(\mm[S])) \cdot \sigma_1(\pi_\odot) \cdot \sigma_1(\hset(\mm[T])^\tpose \hVV)\\
        &\leq \sigma_k(\hUU^\tpose \hset(\mm[S])) \cdot ^\cdot k2^k.
    \end{align*}
    We conclude that $\kappa(\hUU^\tpose \hset(\mm[S])) = \sigma_1(\hUU^\tpose \hset(\mm[S]))/\sigma_k(\hUU^\tpose \hset(\mm[S])) \leq 2 \cdot (k2^k)^2 / (\pimin \bsig^2)$, and the result follows analogously for $\kappa(\hVV^\tpose \hset(\mm[T]))$.
\end{proof}

The following result from~\cite{Bhaskara14} provides us with error bounds for the core step of the tensor decomposition algorithm we are using. 

\begin{theorem}[\cite{Bhaskara14}, Theorem 2.3] \label{thm:decompose}
Let $\eps > 0$ and $\Talg, \hat{\Talg}$ be two $k \times k \times 2$-tensors, such that 
\begin{itemize}
    \item $\Talg = [\XX, \YY, \ZZ]$ with $\XX, \YY \in \R^{k \times k}, \ZZ \in \R^{2 \times k}$;
    \item $\kappa(\XX), \kappa(\YY) \leq \kappa$;
    \item the entries of $(\ZZ_{1k}\ZZ_{2k}^{-1})_k$ are $\zeta$-separated;
    \item $||\XX_i||_2, ||\YY_i||_2, ||\ZZ_i||_2$ are bounded by a constant;
    \item $||\hat{\Talg} - \Talg||_\infty < \eps \cdot \poly(1/\kappa, 1/k, \zeta)$;
\end{itemize}
then the eigenvectors of $\Talg(:,:,1) \Talg(:,:,0)^{-1}$ and $\Talg(:,:,1)^\tpose (\Talg(:,:,0)^\tpose)^{-1}$ approximate the columns $\XX_i$ and $\YY_i$ respectively, up to permutation and additive error $\eps$.
\end{theorem}
(\emph{Comment:} This is slightly more specific than the theorem statement in the reference, but it easily follows from the general statement.)

Finally, we will make use of the following classical result on perturbations of linear systems: 

\begin{lemma}[\cite{error-analysis}, Section 7.1]  \label{lem:linearsystems}
    Let $\mathbf{A}x = b$ and $(\AA + \bDelta \AA)y = b + \Delta b$, where $||\bDelta \AA|| \leq \gamma ||\AA||$ and $||\Delta b|| \leq \gamma ||b||$, and assume that $\gamma \cdot \kappa(\AA) < 1$. Then, 
    \[\frac{||x - y||_2}{||x||_2} \leq \frac{2 \gamma \cdot \kappa(\AA)}{1 - \gamma \cdot \kappa(\AA)}.\]
\end{lemma}

Now, we have all the tools to prove Theorem~\ref{thm: main}.\\

\begin{proofof}{Theorem~\ref{thm: main}}
Fix $\eps \in (0, \zeta/2)$ and a model $(\pi, \mm) \in \DD_{2k-1, \zeta, \pimin}$ with statistics $\mom$, and suppose that Algorithm~\ref{alg:mp algo} is given approximate statistics $\emom$ with $ ||\mom - \emom||_\infty = \ds \leq \eps \cdot (\pimin \zeta^{k})^C$ for some large constant $C$. Consider the $k \times k \times 2$-tensor $\hat{\Talg}$ that consists of the slices $\hCC_{ST}$ and $\hCC_{ST, 1}$ and $\Talg = \left[\hUU^\tpose \hset(\mm[S]), \hVV^\tpose \hset(\mm[T]), \hset(\mm_1) \pi_\odot\right]$. Define $\kappa = \max\{\kappa(\hUU^\tpose \hset(\mm[S])), \kappa(\hVV^\tpose \hset(\mm[T]))\}$. By Lemma~\ref{lem:req}, we have $\kappa \leq \frac{1}{\pimin} \cdot \left(\frac{1}{\zeta}\right)^{O(k)}$ and $||\hat{\Talg} - \Talg||_\infty \leq \exp(O(k))\ds$. Hence, by Theorem~\ref{thm:decompose}, we get $||\hSS - \hUU^\tpose \hset(\mm[S])||_\infty, ||\hTT - \hVV^\tpose \hset(\mm[T])||_\infty \leq \exp(O(k))\ds \cdot \poly(\kappa, k, 1/\zeta) = \left(\frac{1}{\pimin}\right)^{O(1)} \left(\frac{1}{\zeta}\right)^{O(k)} \ds$, after possibly permuting the columns of $\hSS$ and $\hTT$. Now, $\tilde{\pi}$ is defined via a linear system that is a perturbation of equation~\eqref{eq:pi}. If $\ds$ is small enough, then the conditions of Lemma~\ref{lem:linearsystems} are satisfied with $\gamma = \eps/(4\kappa)$, and we get $||\pi - \tilde{\pi}||_2 \leq \eps$. Similarly, we get $\tilde{\mm_i}$ from a perturbed version of equation~\eqref{eq:mm} and we can use Lemma~\ref{lem:linearsystems} with $\gamma = \eps/(4\kappa \cdot \pimin)$ to deduce $||\mm_i - \tilde{\mm_i}||_2 \leq \eps$. Changing the norms to $||.||_\infty$ incurs at most another factor of $k$, by which we can decrease $\ds$, so then the output of Algorithm~\ref{alg:mp algo} satisfies $d_\model((\pi, \mm), (\tilde{\pi}, \tilde{\mm})) < \eps$. \\
To prove the second part of the theorem, suppose the model $(\tilde{\pi}, \tilde{\mm})$ has statistics $\emom$, and again $||\mom - \emom||_\infty \leq \eps \cdot (\pimin \zeta^{k})^C$. By Lemma~\ref{lem:interbds}, $\eCC_{ST}$ has rank $k$. Hence, by equation~\eqref{eq:diagonalization1}, $\tilde{\mm}_1$ is the vector of eigenvalues of $\hCC_{ST, 1} \hCC_{ST}^{-1}$. At the same time, by the first part of the theorem, these eigenvalues give an $\eps$-approximation to $\mm$ (after permuting them), and $\eps < \zeta/2$ implies that they must be separated. We essentially proved in section~\ref{sec:tensdecomp} through equations \eqref{eq:diagonalization1} - \eqref{eq:mm} that $\eCC_{ST}$ having rank $k$ and separation of $\tilde{\mm}_1$ are sufficient conditions for Algorithm~\ref{alg:mp algo} to perfectly recover $(\tilde{\pi}, \tilde{\mm})$ given perfect statistics $\emom = \g(\tilde{\pi}, \tilde{\mm})$. But then, the first part of the theorem implies $d_\model((\tilde{\pi}, \tilde{\mm}), (\pi, \mm)) \leq \eps$, as desired. 
\end{proofof}

\begin{proofof}{Corollary~\ref{cor-n-more-3k-3}}
    Suppose we have $n \geq 2k-1$ variables and we know a subset $A \subseteq [n]$ of $2k-1$ $\zeta$-separated variables. Then, Algorithm~\ref{alg:mp algo} can be used to identify $\mm[A]$ and $\pi$ up to error $\eps$ in runtime $\exp(O(k))$ using $(1/\zeta)^{O(k)}(1/\pimin)^{O(1)}(1/\eps)^2$ many samples. For each $i \notin A$, we can then compute 
    \begin{align*}
        \tilde{\mm_i} = \Paren{(\emom(R \cup \Set{i}))_{R\subseteq T}}^{\tpose} \cdot \hVV \cdot {\Paren{\hTT^{\tpose}}}^{-1} \cdot \tilde{\pi}_{\odot}^{-1}
    \end{align*}
    as in  line~\ref{alg:solve for rest end line} of the algorithm (where $T \subseteq A$ is a set of size $k-1$ and $\hTT$ approximates $\hVV^\tpose \hset(\mm[T])$). This takes runtime at most $n \cdot \exp(O(k))$. Given that $||(\emom(R \cup \Set{i}))_{R\subseteq T} - (\mom(R \cup \Set{i}))_{R\subseteq T}||_\infty \leq \eps (\pimin \zeta^k)^C$ for some large enough $C$, the same analysis as for Theorem~\ref{thm: main} asserts that $\tilde{\mm_i}$ is an $\eps$-approximation to $\mm_i$ (here, $\mom$ is the vector of perfect statistics). Hence, computing $\eps$-approximations to $\mm_i$ for all $i \notin A$ requires obtaining $(\eps (\pimin \zeta^k)^C)$-approximations to $(n-|A|) \cdot 2^{k-1}$ entries of the observable moment vector $\mom$. Standard Chernoff bounds and a union bound show that this is possible with $\log n \cdot (1/\zeta)^{O(k)} (1/\pimin)^{O(1)} (1/\eps)^2$ many samples. \\
    If the subset of $2k-1$ $\zeta$-separated variables is not known (but guaranteed to exist), we can simply run Algorithm~\ref{alg:mp algo} to identify $\mm[A]$ for all $\binom{n}{2k-1}$ possible guesses $A$ of this subset (stopping the algorithm when encountering any issues that might occur when $\hCC_{ST}$ is not invertible). Then, we compute the vector of statistics $\gamma(\tilde{\pi}, \tilde{\mm}[A]) = \hset(\tilde{\mm}[A]) \tilde{\pi}$ for each valid output $(\tilde{\pi}, \tilde{\mm}[A])$ and choose the model $(\tilde{\pi}, \tilde{\mm}[A])$ that minimizes the distance $\ds(A) = ||\gamma(\tilde{\pi}, \tilde{\mm}[A]) - \emom[2^A]||_\infty$. All this takes runtime at most $n^{2k} \exp(O(k))$. Let $A^*$ be the correct guess and suppose we have $||\emom[2^{A^*}] - \mom[2^{A^*}]||_\infty \leq \delta$. Then, by Theorem~\ref{thm: main}, we get $d_\model((\tilde{\pi}, \tilde{\mm}[A^*]), (\pi, \mm[A^*])) \leq \delta \cdot (\pimin \zeta^k)^{-C}$ for some large, positive constant $C$. Hence, we have (after possibly permuting $\tilde{\pi}, \tilde{\mm}[A^*]$) that $||\pi - \tilde{\pi}||_\infty \leq \delta \cdot (\pimin \zeta^k)^{-C}$ and $||\hset(\tilde{\mm}[A^*]) - \hset(\mm[A^*])||_\infty \leq 2^{2k} \delta \cdot (\pimin \zeta^k)^{-C}$. This implies 
    \begin{align*}
        \ds(A^*) &= ||\gamma(\tilde{\pi}, \tilde{\mm}[A^*]) - \emom[2^{A^*}]||_\infty\\
        &\leq ||\hset(\tilde{\mm}[A^*])\tilde{\pi} - \hset(\mm[A^*])\pi||_\infty + ||\mom[2^{A^*}] - \emom[2^{A^*}]||_\infty\\
        &\leq ||\hset(\tilde{\mm}[A^*])\tilde{\pi} - \hset(\mm[A^*])\tilde{\pi}||_2 + ||\hset(\mm[A^*])\tilde{\pi} - \hset(\mm[A^*])\pi||_2  + \delta \\
        &\leq ||\hset(\tilde{\mm}[A^*]) - \hset(\mm[A^*])||_2 \cdot ||\tilde{\pi}||_2 + ||\hset(\mm[A^*])||_2 \cdot ||\tilde{\pi} - \pi||_2 + \delta\\
        &\leq 2^k ||\hset(\tilde{\mm}[A^*]) - \hset(\mm[A^*])||_\infty \cdot \sqrt{k} + 2^k \cdot \sqrt{k} ||\tilde{\pi} - \pi||_\infty + \delta\\
        &\leq 2^{4k} \delta \cdot (\pimin \zeta^k)^{-C}.
    \end{align*}
    If $\delta < 2^{-4k} (\pimin\zeta^k)^{2C}\eps$, then the minimal distance $\ds(A)$ is at most $\eps (\pimin \zeta^k)^C$. Hence, by Theorem~\ref{thm: main}, the model $(\tilde{\pi}, \tilde{\mm}[A])$ is an $\eps$-approximation of $(\pi, \mm[A])$ and we can proceed as in the first part of the proof to get a full $\eps$-approximation to $(\pi, \mm)$. Ensuring that $||\emom[2^A] - \mom[2^A]||_\infty \leq 2^{-4k} (\pimin\zeta^k)^{2C}\eps$ for all subsets of $A \subseteq [n]$ of size $2k-1$ requires $\log(n^{2k}) \cdot  2^{8k} (\pimin \zeta^k)^{-4C} \eps^{-2} = \log n \cdot (1/\pimin)^{O(1)} (1/\zeta)^{O(k)} (1/\eps)^2$ many samples.
\end{proofof}

\section{Lower bounds} \label{sec:lowerbounds}

The following theorem shows that our algorithmic results are optimal when $\zeta$ is small enough. 

\begin{theorem} \label{thm:lowerbd}
    Let $n = 2k-1, \zeta \leq \frac{1}{8k}, \pimin \leq \frac{1}{4k},  \eps > 0$ and $\eps < \min\{\frac{\pimin}{4\sqrt{k}}, \zeta\}$. Then, there exist models $(\pi, \mm)$, $(\pi', \mm') \in \DD_{n, \zeta, \pimin}$ such that $d_\stat(\g(\pi, \mm), \g(\pi', \mm')) \leq (k \zeta)^{\Omega(k)} \eps$, but $d_\model((\pi, \mm), (\pi', \mm')) > \eps$.
\end{theorem}

\begin{corollary} \label{cor:lowerbound}
For $n=2k-1$ random variables $X_i$, sample complexity 
\[(1/(k\zeta))^{\Omega(k)} (1/\eps)\]
is necessary to compute a model that  w.h.p.\ satisfies~\eqref{eq:output}.
\end{corollary} 

Note that (assuming $\pimin$ is not smaller than $\zeta^{O(k)}$), the upper and lower sample complexity bounds in Corollary~\ref{cor:upperbd} and Corollary~\ref{cor:lowerbound} match in the case that $\zeta \leq k^{-1-\delta}$ for some arbitrary small $\delta > 0$. Only when $\zeta$ comes closer to its maximal possible value of $\frac{1}{k-1}$, there is a gap between upper and lower bound. In the edge case, when $\zeta = \Theta(\frac{1}{k})$, the upper bound evaluates to $\exp{(O(k \log k))}$ and the lower bound evaluates to $\exp{(\Omega(k))}$. We remark that the lower bound can be shown to be tight for $k$-MixIID over the entire range of the parameter  $\zeta$.

To prove Theorem~\ref{thm:lowerbd}, we start with the following lemma:

\begin{lemma} \label{lem:lowerbound}
    Let $n \geq 1, \eps > 0$ and $\eps < \min\{\frac{\pimin}{4\sqrt{k}}, \zeta\}$. Suppose $(\pi, \mm) \in \DD_{n,\zeta,\pimin}$ and $\sigma_k(\hset(\mm)) = \sigma < \frac{1}{2}$. Then, there exists $\hat{\pi}$ with $\min_j \hat{\pi}_j \geq \frac{1}{4} \pimin$ and such that $d_\model((\pi, \mm), (\hat{\pi}, \mm)) > \eps$ but $d_\stat(\gamma(\pi, \mm), \gamma(\hat{\pi}, \mm)) \leq 4k\sigma \cdot \eps$.
\end{lemma}

\begin{proof}
    Let $(\pi, \mm) \in \DD_{n,\zeta,\pimin}$ and $\sigma_k(\hset(\mm)) = \sigma < \frac{1}{2}$. Let $\tilde{\hset(\mm)}$ be the best rank-$(k-1)$-approximation of $\hset(\mm)$. By the Eckart-Young Theorem~\cite{EckartYoung36}, we have $||\tilde{\hset(\mm)} - \hset(\mm)|| = \sigma$. Let $\alpha \in \R^k$ be a vector in the right kernel of $\tilde{\hset(\mm)}$ with $||\alpha||_2 = 1$. Let $\mathbb{1}$ denote the all-ones vector in $\R^k$ and let $e_1$ denote the vector whose first entry is 1 and all other entries are zero. We have 
    \begin{align*}
        |\mathbb{1}^\tpose \alpha| = |(\hset(\mm) \alpha)_1| = |((\hset(\mm) - \tilde{\hset(\mm)})\alpha)_1| \leq ||\hset(\mm) - \tilde{\hset(\mm)}|| \cdot ||\alpha||_2 \leq \sigma.
    \end{align*}
    Now, define $\hat{\pi} = \pi + 2 \sqrt{k}\eps \cdot (\alpha - (\mathbb{1}^\tpose \alpha)e_1)$. First, we check that $\hat{\pi}$ is a valid probability vector. By our assumptions, we have 
    \begin{align*}
        ||2 \sqrt{k}\eps \cdot (\alpha - (\mathbb{1}^\tpose \alpha)e_1)||_\infty \leq 2 \sqrt{k}\eps \cdot (||\alpha||_\infty + |\mathbb{1}^\tpose \alpha|) < \frac{\pimin}{2} \cdot (1 + \sigma) \leq \frac{3}{4}\pimin,
    \end{align*}
    hence, all the entries of $\hat{\pi}$ are larger than $\frac{1}{4}\pimin$. Moreover, we have 
    \begin{align*}
        \sum_{j=1}^k \hat{\pi}_j = \mathbb{1}^\tpose \pi +  2\sqrt{k}\eps \cdot (\mathbb{1}{^\tpose} \alpha - \mathbb{1}^\tpose \alpha) = 0. 
    \end{align*}
    Now, recall the definition of $d_\model$ as 
    \begin{align*}
        d_\model ((\pi,\mm),(\hat{\pi},\mm)) =\min_{\rho \in S_k} \max\{\max_j |\pi_j-\hat{\pi}_{\rho(j)}|, \max_{i,j} |\mm_{i,j}-\mm_{i,\rho(j)}|\}.
    \end{align*}
    For any permutation $\rho$ that is not the identity, we have $\max_{i,j} |\mm_{i,j}-\mm_{i,\rho(j)}| \geq \zeta > \eps$ by $\zeta$-separation of $\mm$. For the identity permutation, we have 
    \begin{align*}
        \max_j |\pi_j-\hat{\pi}_j| = ||2\sqrt{k}\eps \cdot (\alpha - (\mathbb{1}^\tpose \alpha)e_1)||_\infty \geq 2\eps \cdot (||\alpha||_2 - ||(\mathbb{1}^\tpose \alpha)e_1||_2) \geq 2\eps \cdot (1 - \sigma) > \eps,
    \end{align*}
    so we conclude $d_\model((\pi, \mm), (\hat{\pi}, \mm)) > \eps$. Moreover, we have 
    \begin{align*}
        d_\stat(\gamma(\pi, \mm), \gamma(\hat{\pi}, \mm)) &= ||\hset(\mm) \pi - \hset(\mm) \hat{\pi}||_\infty\\ &= 2\sqrt{k} \eps \cdot ||\hset(\mm) (\alpha - (\mathbb{1}^\tpose \alpha) e_1)||_\infty\\
        &\leq 2\sqrt{k}\eps \cdot \left(||\hset(\mm) \alpha||_\infty + |\mathbb{1}^\tpose \alpha| \cdot ||\hset(\mm)||_\infty\right)\\
        &\leq 2\sqrt{k}\eps \cdot \left(\sqrt{k} \cdot ||(\hset(\mm) - \tilde{\hset(\mm)}) \alpha||_2 + \sigma \right)\\
        &\leq 2\sqrt{k} \eps \cdot \left(\sqrt{k} \sigma + \sigma\right) \leq 4k\sigma \cdot \eps.
    \end{align*}
\end{proof}

Lemma~\ref{lem:lowerbound} reduces the challenge of finding models with large model distance but small statistical distance to finding a model with a small $k$'th singular value of the Hadamard extension $\hset(\mm)$. To do this, we will need the following known result on Vandermonde matrices (also see Definition~\ref{defn:Vdm}). 

\begin{lemma}[\cite{error-analysis}, section 22.1]\label{lem:vdm}
    For any row vector $m \in \R^k$, we have 
    \begin{align*}
        ||\Vandermonde(m)^{-1}||_\infty \geq \max_i \prod_{j \neq i} \frac{\max\{1,|m_j|\}}{|m_i - m_j|}.
    \end{align*}
\end{lemma}

\begin{lemma}\label{lem:sigma-upperbound}
    Let $n \geq k-1$. For any $\zeta \leq \frac{1}{k}$, there exists a matrix $\mm \in [0,1]^{n \times k}$ with $\zeta$-separated columns and $\sigma_k(\hset(\mm)) \leq n2^n \cdot (k\zeta)^k$.
\end{lemma}

\begin{proof}
    Given $\zeta$, define $\mm$ as the matrix with $n$ identical rows of the form $\mm_i = (0, \zeta, 2\zeta, \dots, (k-1)\cdot\zeta)$. According to Lemma~\ref{lem:vdm}, we have 
    \begin{align*}
        \sigma_k(\Vandermonde(\mm_1)) = ||\Vandermonde(\mm_1)^{-1}||^{-1} \leq k ||\Vandermonde(\mm_1)^{-1}||_\infty^{-1} \leq k \cdot (k-1)! \zeta^{k-1} = k! \zeta^{k-1}.
    \end{align*}
    Now, consider the matrix $\Vandermonde(\mm_1, n) = \binom{\Vandermonde(\mm_1)}{\mathbf{R}}$, where $\mathbf{R} \in \R^{(n-k) \times k}$ denotes the last $n-k$ columns of $\Vandermonde(\mm_1, n)$. Note that $||\mathbf{R}|| \leq n \cdot ||\mathbf{R}||_\infty \leq n \cdot ((k-1) \cdot \zeta)^k$. Let $q \in \R^k, ||q||_2 = 1$ such that $||\Vandermonde(\mm_1)q||_2$ is minimized. We have 
    \begin{align*}
        ||\Vandermonde(\mm_1, n)q||_2 \leq ||\Vandermonde(\mm_1)q||_2 + ||\mathbf{R}q||_2 \leq \sigma_k(\Vandermonde(\mm_1)) + ||\mathbf{R}|| \leq n \cdot k^k \cdot \zeta^k,
    \end{align*}
    hence, $\sigma_k(\Vandermonde(\mm_1, n)) \leq n \cdot k^k \cdot \zeta^k$. Finally, note that $\hset(\mm)$ is a matrix with $2^n$ rows that are all duplicates of some row in $\Vandermonde(\mm_1, n)$. Hence, we have $\sigma_k(\hset(\mm)) \leq n2^n \cdot (k\zeta)^k$.
\end{proof}

\begin{proofof}{Theorem~\ref{thm:lowerbd}}
Fix $n, \zeta, \pimin$, and $\eps$ as in the statement of Theorem~\ref{thm:lowerbd}. Let $\pi := (1/k, \dots, 1/k)^\tpose$. According to Lemma~\ref{lem:sigma-upperbound}, there exists $\mm$ such that $(\pi, \mm) \in \DD_{n, \zeta, \pimin}$ and $\sigma_k(\hset(\mm)) \leq n2^n \cdot (k\zeta)^k = (2k-1)2^{2k-1} \cdot (k\zeta)^k < \frac{1}{2}$. Now, by Lemma~\ref{lem:lowerbound}, there exists $\hat{\pi}$ such that $(\hat{\pi}, \mm) \in \DD_{n, \zeta, \pimin}, d_\model((\pi, \mm), (\hat{\pi}, \mm)) > \eps$ and $d_\stat(\gamma(\pi, \mm), \gamma(\hat{\pi}, \mm)) \leq 4k \sigma_k(\hset(\mm)) \cdot \eps \leq (k\zeta)^{\Omega(k)} \cdot \eps$. 
\end{proofof}

\section{Discussion} \label{sec:discussion}
Two larger questions remain to be addressed in this area. First, it would be of great interest to achieve a similar sample complexity as in Theorem~\ref{thm: main} for the more general ``learning'' task. Second, it is open to characterize the set of models, for which identification (even with perfect statistics) is possible.
There exist identifiable models $(\pi, \mm)$ where none of the rows of $\mm$ has fully-separated entries. Recent work~\cite{gordon2022hadamard} gives a sufficient condition for identification that is less restrictive (though more complicated) than $\zeta$-separation. However, there is no known way of obtaining in that less restrictive framework the quantitative bounds on noise-stability, which are essential to this paper. (Quantification in that framework would likely be misguided, anyway, given that it is a complex yet not tight characterization; for example it excludes mixtures of subcubes~\cite{ChenMoitra19}.) Note that some kind of separation assumption is unavoidable if we insist on $\mathcal{L}_\infty$-reconstruction of the model parameters: as one example, if there are $j,j'$ s.t.\ $\mm_{ij} = \mm_{ij'}$ for all $i$, then it is impossible to determine $\pi_j$ and $\pi_{j'}$. However, it might be possible to completely eliminate the separation assumption in favor of settling for reconstruction in transportation (Wasserstein) distance. This was achieved for $k$-MixIID (the $k$-MixProd problem where all observables are conditionally identically distributed) in~\cite{LRSS15}, and improved in~\cite{FanLi22}. It is an open question whether these ideas can be extended to $k$-MixProd, with the goal being transportation-cost reconstruction of each of the rows of $\mm$. 

We remark that even though Theorem~\ref{thm: main} and its corollaries are given for $\zeta$-separated models, our algorithm itself also works under weaker conditions. In fact, all it needs is just one $\zeta$-separated observable and two more disjoint sets of observables, each of which have a Hadamard extension (see below) with good condition number. These requirements can, for models in ``general position,'' be met with as few as $2 \lg k + 1$ observables. The sample complexity and runtime of the algorithm scale singly-exponentially in the number of observables actually used. Hence, it is possible that, except for a small set of ``adversarial" models, model identification typically be achieved with much lower sample complexity and runtime than can be guaranteed for the worst case. In fact, this line of thought has already been pursued for the tensor decomposition problem to which we reduce $k$-MixProd. It was shown that the time complexity of the tensor decomposition problem significantly improves when each tensor component of the input tensor is perturbed by small random noise~\cite{Bhaskara14}. Note that one can never solve the $k$-MixProd Identification problem  with fewer than $\lg k$ observables, because then the Hadamard extensions cannot be full rank, and then the mapping from model to statistics, cannot be injective.

\bibliographystyle{plainurl}
\bibliography{refs}

\end{document}